\documentclass[11pt,letterpaper]{mystyle}

\usepackage[comma,authoryear,compress]{natbib}
\usepackage{graphicx}
\usepackage{subcaption}
\usepackage{booktabs}
\usepackage{amsmath}
\usepackage{amssymb}
\usepackage{amsthm}
\usepackage{algorithm}
\usepackage{algorithmic}
\usepackage{placeins}
\usepackage{listings}
\tcbuselibrary{skins,breakable,listings}

\usepackage{amsmath,amsfonts,bm}

\def\eqref#1{equation~\ref{#1}}

\def\1{\bm{1}}

\DeclareMathAlphabet{\mathsfit}{\encodingdefault}{\sfdefault}{m}{sl}
\SetMathAlphabet{\mathsfit}{bold}{\encodingdefault}{\sfdefault}{bx}{n}

\theoremstyle{plain}
\newtheorem{proposition}{Proposition}

\newtheorem{assumption}{Assumption}
\newtheorem{theorem}{Theorem}

\definecolor{promptframe}{RGB}{31,56,100}
\newtcblisting{promptbox}[1]{
  listing only,
  breakable,
  enhanced,
  colback=white,
  colframe=promptframe,
  boxrule=0.5pt,
  arc=1pt,
  title={#1},
  fonttitle=\bfseries\small,
  left=4pt,
  right=4pt,
  top=2pt,
  bottom=2pt,
  listing options={
    basicstyle=\ttfamily\scriptsize,
    breaklines=true,
    breakatwhitespace=true,
    columns=fullflexible,
    keepspaces=true,
    showstringspaces=false,
    upquote=true,
    frame=none,
    aboveskip=0pt,
    belowskip=0pt
  }
}

\title{\LARGE UOPD: Uncertainty-Aware Intervention for \\
On-Policy Distillation of Multi-Turn Agents}
\runningtitle{UOPD: Uncertainty-Aware Intervention for On-Policy Distillation of Multi-Turn Agents}

\author{\textbf{Wenbo Zhang}\textsuperscript{1,*} \quad
\textbf{Pengcheng Xu}\textsuperscript{1,*} \quad
\textbf{Weizhi Du}\textsuperscript{2} \quad
\textbf{Jing Zhang}\textsuperscript{1} \quad
\textbf{Hengrui Cai}\textsuperscript{1} \\
\textsuperscript{1}University of California, Irvine \quad
\textsuperscript{2}University of Michigan, Ann Arbor
}

\hypersetup{
  pdftitle={UOPD: Uncertainty-Aware Intervention for On-Policy Distillation of Multi-Turn Agents},
  pdfauthor={Wenbo Zhang, Pengcheng Xu, Weizhi Du, Jing Zhang, Hengrui Cai}
}

\begin{document}

\begin{abstract}
On-policy distillation (OPD) trains a student on its own rollouts using dense supervision from a teacher. In multi-turn environments, a mistake at a critical decision step can redirect the subsequent rollout toward poor outcomes. We use low teacher confidence on student actions to select high-uncertainty steps for correction. In a controlled ALFWorld study, a single teacher correction at a low-confidence step improves subsequent student behavior and task success, motivating selective intervention during distillation. We propose \textbf{UOPD}, an uncertainty-aware intervention method for on-policy distillation. At low-uncertainty turns, UOPD executes student actions and applies the standard OPD loss. At high-uncertainty turns, it samples and executes teacher actions and trains the student to imitate them through supervised fine-tuning, which minimizes forward Kullback–Leibler divergence in expectation. UOPD utilizes adaptive uncertainty thresholds to target a scheduled intervention rate. Empirically, we evaluate UOPD across a broad range of agentic tasks, including ALFWorld, WebShop, and Search, demonstrating its superior performance over OPD methods and their variants. UOPD improves WebShop score by up to $15.8\%$ relative to standard OPD.

\vspace{6mm}
\textbf{Date:} September 26, 2026

\raisebox{-0.15\height}{\includegraphics[height=1em]{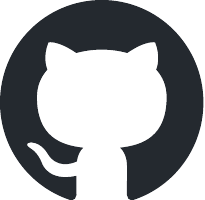}}
\textbf{Code:}
\href{https://github.com/onepounchman/UOPD}{https://github.com/onepounchman/UOPD}

\textbf{Author Emails:}
\href{mailto:wenbz13@uci.edu,pengchx3@uci.edu,hengrc1@uci.edu}{\{wenbz13,pengchx3,hengrc1\}@uci.edu}
\end{abstract}

\maketitle
\begingroup
\renewcommand{\thefootnote}{*}
\footnotetext{These authors contributed equally to this work.}
\endgroup

\section{Introduction}
\label{sec:intro}

\begin{figure}[t]
\centering
\includegraphics[width=\linewidth]{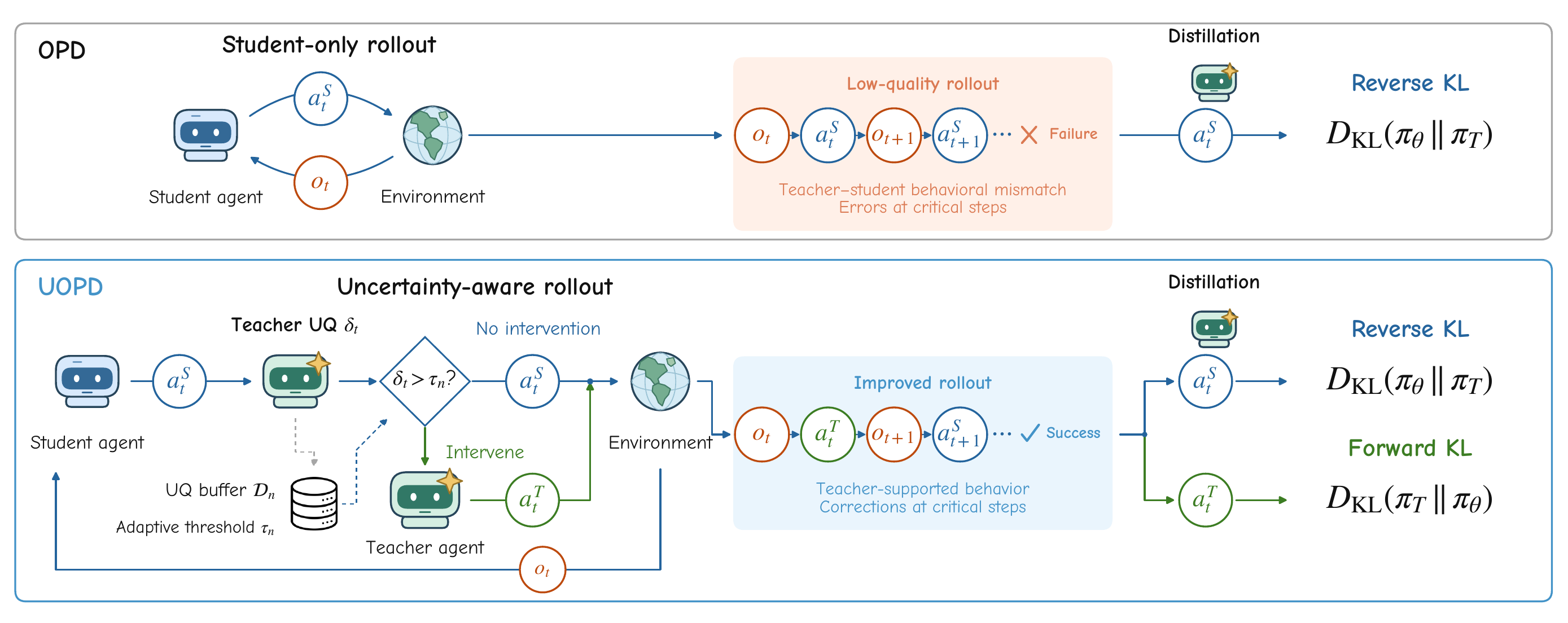}
\caption{\textbf{OPD versus UOPD.} OPD executes student actions ($a^S_t$) given observations ($o_t$) and distills using reverse KL. UOPD uses teacher signals for uncertainty quantification (UQ) and an adaptive threshold to determine when the teacher intervenes. This decision jointly determines the executed action and the learning objective: reverse KL for student actions, or SFT  on teacher actions ($a^T_t$). The latter minimizes forward KL in expectation.}
\label{fig:opd_comparison}
\vspace{-3mm}
\end{figure}

Distilling a large teacher into a small student has become a standard stage of language-model post-training \citep{hinton2015distilling,agarwal2024policy}. Among its variants, \emph{on-policy distillation} (OPD) has emerged as the dominant recipe: rather than imitating a fixed corpus of teacher trajectories, the student samples its own rollouts and the teacher supplies a dense, token-level target at every state the student actually visits \citep{agarwal2024policy}. This removes the exposure bias of supervised distillation and, by minimizing a reverse Kullback–Leibler (KL) under the student's own distribution, gives the student a mode-seeking target that it has the capacity to match.  Recent work has extended OPD to multi-turn agents, which gather
  information, use tools, and act on successive observations to accomplish
  complex tasks \citep{wang2026exploring,sod2026,sdar2026,sageopd2026}.
  Distilling these capabilities into smaller models is particularly
  valuable for efficient deployment, as completing a single task often
  requires many model calls.

In multi-turn settings, each action shapes subsequent observations and decisions. Student errors can therefore compound across turns, degrading trajectory quality and teacher supervision quality \citep{wang2026exploring,sod2026,manyfaces2026}. Standard OPD can reduce the probability of sampled poor actions through distillation, but does not directly provide better actions for the student to imitate. These poor actions are still executed during rollout collection and can lead to repeated mistakes and task failure. Existing approaches mitigate multi-turn drift by reweighting teacher supervision \citep{sod2026,sageopd2026}, adapting rollout length \citep{pruneopd2026,wang2026exploring}, or using teacher-generated prefixes \citep{wang2026exploring,reopd2026}. However, these approaches do not directly correct unreliable student
  actions at the turns where they arise. Executing such actions can lead
  to further mistakes and task failure, reducing the quality of training
  trajectories. 
  
  To address this gap, we consider \textit{intervention}: replacing selected student actions before execution. As the teacher is generally more capable than the student, it can supply a replacement action for such an intervention, providing explicit corrective targets while guiding the subsequent rollout. We focus on deciding when such intervention is needed. The student may generate actions that elicit high uncertainty from the teacher. This uncertainty offers a potential signal for identifying steps where a teacher correction could be useful.   Our controlled study in Section~\ref{sec:motivation} shows that a single
  teacher correction yields a larger improvement in task success when
  applied at a high-uncertainty turn than at a randomly selected turn.
  These results support using teacher uncertainty to select turns
  for intervention.

Motivated by these findings, we propose \textbf{Uncertainty-Aware Intervention for On-Policy Distillation (UOPD)} (Fig.~\ref{fig:opd_comparison}). At each turn, the student proposes an action, and the teacher evaluates its uncertainty on that action. At low-uncertainty turns, UOPD executes the student action and trains with reverse KL. At high-uncertainty turns, it samples and executes a teacher action and trains the student to imitate it through supervised fine-tuning (SFT), which minimizes forward KL in expectation. Teacher intervention thus supplies an action that the student may rarely generate on its own and lets the student continue interacting from the resulting observation. UOPD uses uncertainty to decide when to learn from the student's own actions and when to introduce a teacher correction.

Our \textbf{contributions} are summarized as follows:

$\bullet$  We propose \textbf{UOPD}, which uses teacher uncertainty to guide
  action execution and supervision in multi-turn on-policy distillation.
  UOPD combines reverse-KL learning on student actions with SFT on executed
  teacher corrections, providing corrective targets and improving
  trajectory quality.\\
$\bullet$ We provide a theoretical analysis of the two roles of teacher corrections.
We derive a rollout return improvement guarantee under average recoverability,
and show that teacher-action imitation can improve student return even when
reverse-KL supervision reduces it.\\
$\bullet$  Experiments on ALFWorld, WebShop, and multi-hop Search demonstrate that UOPD improves task performance over OPD and other state-of-the-art baselines across student model sizes while requiring fewer interaction turns, indicating the effectiveness and efficiency of our method.
\section{Related Work}
\label{sec:related}

\textbf{On-policy distillation.}
On-policy distillation trains a student on its own generations while querying a teacher for dense token-level supervision \citep{agarwal2024policy}. Recent work connects this objective to policy-gradient learning \citep{vopd2026,exopd2026}, organizes distillation methods by their rollout source and divergence direction \citep{decoupled2026}, and studies how the choice of divergence affects policy improvement \citep{distil2026}. Other analyses show that the effectiveness of OPD depends strongly on student--teacher support overlap and that teacher supervision can become unreliable on student-generated prefixes \citep{manyfaces2026,rethinking2026}. Several methods address this issue at the token level by adapting the KL direction \citep{aopd2026,hpd2026}, selecting supported tokens \citep{taopd2026}, or mixing teacher and student prefixes \citep{xu2025skd,brts2026}. Orthogonal to where supervision is applied, a further group of variants changes what the teacher provides: rubrics or preference pairs induced from teacher text stand in for logits when the teacher is black-box \citep{ropd2026,orpodistill2025}, privileged context conditions the teacher's scores \citep{mopd2026}, and another line couples teacher supervision with an explicit task reward inside a single objective \citep{kdrl2025,rlad2026,sparse2dense2026,chord2026}, following earlier on-policy policy distillation for control \citep{ppd2025}. Teacher uncertainty is also used to adapt token-level distillation objectives \citep{eopd2026,egrsd2026}, while other work obtains corrective supervision by probing alternative continuations or backtracking to an earlier reasoning prefix \citep{spot2026,motab2026}. UOPD selects corrections during environment interaction, where executing a corrected action also changes the states visited next.

\textbf{Multi-turn agent distillation and intervention.}
In multi-turn environments, errors alter subsequent observations and cause distribution shift to accumulate across turns \citep{wang2026exploring,sod2026,sageopd2026}. Existing approaches mitigate this drift by reweighting or filtering supervision \citep{sod2026,pruneopd2026,sdar2026}, limiting or adapting rollout exposure \citep{wang2026exploring,fullrollouts2026,adwin2026,turnopd2026}, or modifying the roll-in trajectory through teacher prefixes, refinement, or lookahead feedback \citep{reopd2026,trd2026,lgr2026,serl2026}. SAGE-OPD \citep{sageopd2026} uses teacher judgments and confidence to scale turn-level supervision. Teacher actions also guide rollouts through scheduled turn mixing \citep{guidedopd2026} or corrections selected and validated after rollout collection \citep{futurebridgeopd2026}. UOPD uses teacher uncertainty to target critical steps and intervene before the student's proposed action is executed, providing a corrective learning target and redirecting the subsequent rollout. UOPD also connects to DAgger-style expert supervision at learner-visited states \citep{ross2011dagger,distil2026} and HG-DAgger's expert takeover \citep{kelly2019hgdagger}. It automates takeover using teacher uncertainty on student actions and trains the student on the executed corrections.

\section{Why Is Intervention Needed Along the Student's Rollout?}
\label{sec:motivation}

\begin{figure}[!ht]
\centering
\captionsetup[subfigure]{font=scriptsize,skip=2pt}
\begin{subfigure}[t]{0.24\linewidth}
  \centering
  \includegraphics[width=\linewidth]{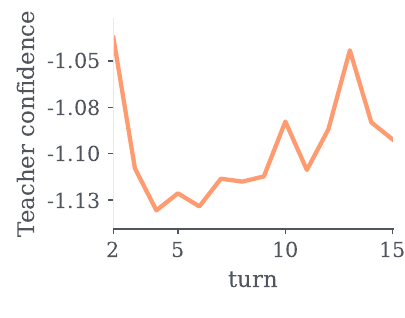}
  \caption{Student rollout}
\end{subfigure}\hfill
\begin{subfigure}[t]{0.24\linewidth}
  \centering
  \includegraphics[width=\linewidth]{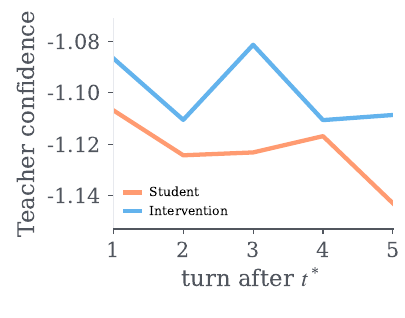}
  \caption{After intervention}
\end{subfigure}\hfill
\begin{subfigure}[t]{0.24\linewidth}
  \centering
  \includegraphics[width=\linewidth]{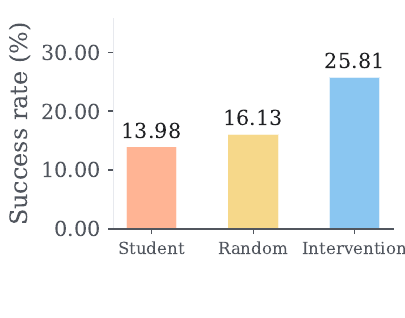}
  \caption{Intervention benefit}
\end{subfigure}\hfill
\begin{subfigure}[t]{0.24\linewidth}
  \centering
  \includegraphics[width=\linewidth]{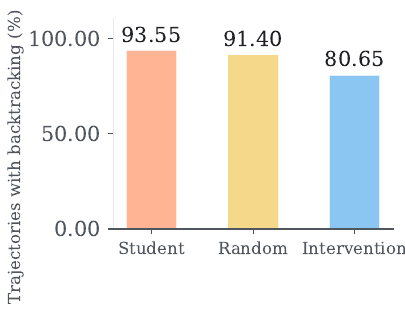}
  \caption{Backtracking behavior}
\end{subfigure}

\caption{\textbf{One teacher intervention at a low-confidence step improves rollout quality.} (a) Teacher confidence fluctuates across turns of student-only rollouts. (b) Replacing one low-confidence action with a teacher action raises teacher confidence on subsequent actions generated by the student. (c)(d) This intervention increases task success and reduces backtracking compared with both student-only rollouts and random intervention. These results support using teacher uncertainty to select steps where a correction improves the rest of the rollout.}
\label{fig:motivation}
\end{figure}

 A poor student action can lead to further mistakes, affecting both task
  completion and the trajectory used for distillation. Standard OPD provides
  teacher supervision along this trajectory, but does not prevent the action
  from being executed. Correcting the action before execution can instead
  allow the student to continue from a better state and learn from the
  resulting trajectory. We therefore study whether teacher uncertainty
  identifies useful intervention points and whether a single teacher
  correction improves the student's subsequent rollout and task performance.

\noindent{\textbf{Experimental Setup.}}
We conduct a controlled study on ALFWorld \citep{shridhar2021alfworld}, keeping both models fixed. We use teacher confidence on student-generated actions to identify high-uncertainty steps. On the same tasks, we compare student-only rollouts, intervention at a low-confidence step, and intervention at a random step. Both intervention conditions replace exactly one student action with a teacher action, then return control to the student under the same turn limit. Appendix~\ref{app:motivation_details} provides the full protocol.

\noindent{\textbf{Observation 1: uncertainty varies across decision steps.}}
Teacher confidence falls and recovers along student rollouts (Fig.~\ref{fig:motivation}(a)), showing that the teacher's support for student-generated actions changes within an interaction. Low-confidence actions are weakly supported by the teacher and provide candidates for correction. In a multi-turn environment, executing such an action also determines the observations and decisions that follow, so its consequences can extend beyond the current turn. This motivates evaluating the student's proposed action in its current interaction history to decide when intervention may be useful.

\noindent{\textbf{Observation 2: one correction improves subsequent student behavior.}}
After intervention, the teacher assigns higher confidence to subsequent student-generated actions (Fig.~\ref{fig:motivation}(b)). The teacher acts only once, and the student resumes without any parameter update: the benefit extends to actions the student generates on its own after the correction. Task success also rises from approximately $14\%$ to $26\%$, compared with $16\%$ for random intervention (Fig.~\ref{fig:motivation}(c)). The two intervention conditions use the same amount of teacher assistance, yet selecting a low-confidence step produces a larger gain. These results show that both the action correction and its timing matter for the subsequent rollout.

\noindent{\textbf{Observation 3: intervention helps correct student errors.}}
Student errors can alter subsequent interaction and lead to behavior that fails to advance the task. We examine this effect in ALFWorld using returns to previously visited states as a task-specific behavioral indicator. Fig.~\ref{fig:motivation}(d) shows that low-confidence intervention reduces the fraction of trajectories that leave a state and later return to it, relative to both student-only rollouts and random intervention. Together with higher task success, this change suggests that a teacher correction can help the student recover from mistakes and make progress toward completing the task. A well-placed teacher correction can thus improve trajectory quality by changing how the student continues to act.

\noindent{\textbf{Takeaway.}}
Teacher correction at a high-uncertainty step improves subsequent student behavior and task success. Its advantage over random intervention shows that correction timing matters, motivating teacher uncertainty as a signal for deciding when to intervene.

\section{Uncertainty-Aware Intervention for On-Policy Distillation}
\label{sec:method}

\subsection{Notation and Preliminaries}
\label{subsec:preliminary}

We consider a multi-turn interaction starting from a task input $q$. Before turn $t$, the agent conditions on the interaction history $h_t=(q,a_0,o_1,\ldots,a_{t-1},o_t)$, with $h_0=q$. A policy $\pi$ samples an action response $a_t\sim\pi(\cdot\mid h_t)$. Executing the action in the environment produces an observation $o_{t+1}$ and updates the history to $h_{t+1}=(h_t,a_t,o_{t+1})$. A trajectory of $H$ turns is $\tau=(q,a_0,o_1,\ldots,a_{H-1},o_H)$.

Let $\pi_\theta$ be the student policy and $\pi_T$ a frozen teacher. In standard OPD, the student generates every action, $a_t\sim\pi_\theta(\cdot\mid h_t)$. The teacher provides dense token-level supervision on the resulting rollouts. The OPD objective minimizes the reverse Kullback--Leibler (KL) divergence from student to teacher at the visited histories:
\begin{equation}
\mathcal L_{\mathrm{OPD}}(\theta)
=\mathbb E_{\tau\sim\pi_\theta}\!\left[
\sum_{t=0}^{H-1}
\mathrm{KL}\!\left(
\pi_\theta(\cdot\mid h_t)
\,\middle\|\,
\pi_T(\cdot\mid h_t)
\right)
\right].
\label{eq:opd_loss}
\end{equation}
For a sampled student action $a_t$, we write $\ell_t^{\mathrm{OPD}}(a_t;\theta):=\log\pi_\theta(a_t\mid h_t)-\log\pi_T(a_t\mid h_t)$. Its expectation over $a_t\sim\pi_\theta(\cdot\mid h_t)$ equals the KL divergence at $h_t$ in Eq.~\ref{eq:opd_loss}.

Motivated by the observations in Section~\ref{sec:motivation}, UOPD selectively replaces student actions with teacher corrections during rollout collection.

\subsection{UOPD: Uncertainty-Aware Intervention for On-Policy Distillation}
\label{subsec:said}

UOPD consists of two components: an intervention rule that decides when to replace a student action with a teacher action, and a loss that specifies how the student learns at each turn.

  \noindent{\textbf{Intervention Rule.}}
  Let $\mathcal U(\pi_T,\pi_\theta;h,a)\in\mathbb R$ be an uncertainty
  score for action $a$ given history $h$, with larger values indicating
  greater uncertainty. The uncertainty score can be instantiated using
  teacher confidence or a student--teacher confidence gap; we provide
  the specific definitions in Appendix~\ref{app:training-config}.

  Let $n$ index batches of episodes collected during training,
  and let $r_n$ be the target intervention rate for batch $n$.
  Given a buffer $\mathcal D_n$ of recent uncertainty scores of student
  actions, UOPD sets the intervention threshold to
  $\tau_n=\operatorname{Quantile}_{1-r_n}(\mathcal D_n)$.
  The target rate controls how much the teacher
  intervenes, while uncertainty determines where it intervenes.
  We recompute the threshold at the start of each episode and hold it
  fixed within the episode.
  This design connects to uncertainty-guided expert intervention
  \citep{menda2019ensembledagger} and budget-aware quantile thresholding
  \citep{hoque2022thriftydagger} in interactive imitation learning.

  At turn $t$, the student samples a proposed action
  $a_t^S\sim\pi_\theta(\cdot\mid h_t)$ and UOPD computes its uncertainty
  score $\delta_t=\mathcal U(\pi_T,\pi_\theta;h_t,a_t^S)$.
  If $\delta_t>\tau_n$, UOPD samples a teacher action
  $a_t^T\sim\pi_T(\cdot\mid h_t)$ and executes it in place of $a_t^S$;
  otherwise, it executes the student action. The executed action $a_t$
  is appended to the history along with the resulting observation.

  We gradually decrease the target intervention rate to provide more
  teacher guidance early in training and progressively shift rollout
  control to the student:
  \begin{equation}
  r_n=r_{\mathrm{start}}
  +(r_{\mathrm{end}}-r_{\mathrm{start}})
  \min\!\left(\frac{n}{N_{\mathrm{decay}}},1\right),
  \label{eq:intervention_schedule}
  \end{equation}
where $r_{\mathrm{start}}$ and $r_{\mathrm{end}}$ denote the initial and final target rates, and
  $N_{\mathrm{decay}}$ denotes the decay duration.

\begin{algorithm}[t]
\caption{UOPD}
\label{alg:said}
{\bfseries Input:} student $\pi_\theta$, frozen teacher $\pi_T$, environment $E$, uncertainty score $\mathcal U$, uncertainty score buffer $\mathcal D_0$, rate parameters $(r_{\mathrm{start}},r_{\mathrm{end}},N_{\mathrm{decay}})$, SFT weight $\beta$.
\begin{algorithmic}[1]
\FOR{each training iteration $n=0,1,\ldots$}
  \STATE Compute target intervention rate $r_n$ using Eq.~\ref{eq:intervention_schedule}.
  \FOR{each episode in the rollout batch, in parallel}
    \STATE Set $\tau_n\gets\operatorname{Quantile}_{1-r_n}(\mathcal D_n)$; hold it fixed for this episode.
    \STATE Reset $E$ for task input $q$; set $h_0\gets q$ and $t\gets0$.
    \WHILE{the episode is not terminated}
      \STATE Sample $a_t^S\sim\pi_\theta(\cdot\mid h_t)$; compute and record $\delta_t=\mathcal U(\pi_T,\pi_\theta;h_t,a_t^S)$.
      \STATE If $\delta_t>\tau_n$, sample $a_t^T\sim\pi_T(\cdot\mid h_t)$ and set $a_t\gets a_t^T$; otherwise set $a_t\gets a_t^S$.
      \STATE Compute $\ell_t^{\mathrm{UOPD}}$ using Eq.~\ref{eq:said_step}.
      \STATE Execute $a_t$, observe $o_{t+1}$, and set $h_{t+1}\gets(h_t,a_t,o_{t+1})$; $t\gets t+1$.
    \ENDWHILE
    \STATE Update $\mathcal D_n$ with this episode's recorded uncertainty scores.
  \ENDFOR
  \STATE Update the student $\pi_\theta$ using the collected batch; carry the updated buffer to $\mathcal D_{n+1}$.
\ENDFOR
\end{algorithmic}
\end{algorithm}

\noindent{\textbf{Intervention-Conditioned Distillation.}} 
At a non-intervened turn, the student receives the standard OPD supervision on $a_t^S$. At an intervened turn, the executed action $a_t^T$ is sampled from the teacher. The standard OPD term relies on student-policy samples and therefore cannot be applied directly to $a_t^T$: at a fixed history $h_t$, its expectation under teacher sampling becomes
\begin{equation}
\mathbb E_{a_t^T\sim\pi_T(\cdot\mid h_t)}\!\left[
\log\frac{\pi_\theta(a_t^T\mid h_t)}{\pi_T(a_t^T\mid h_t)}
\right]
=-\mathrm{KL}\!\left(\pi_T(\cdot\mid h_t)\,\middle\|\,\pi_\theta(\cdot\mid h_t)\right).
\label{eq:teacher_sampled_log_ratio}
\end{equation}
We instead minimize the \emph{forward} KL from teacher to student at intervened turns through SFT, since
\begin{equation*}
\mathbb E_{a_t^T\sim\pi_T(\cdot\mid h_t)}[-\log\pi_\theta(a_t^T\mid h_t)]
=\mathcal{H}\!\left(\pi_T(\cdot\mid h_t)\right)
+\mathrm{KL}\!\left(\pi_T(\cdot\mid h_t)\,\middle\|\,\pi_\theta(\cdot\mid h_t)\right).
\end{equation*}
  Here $\mathcal H(p):=-\mathbb E_{a\sim p}[\log p(a)]$ denotes the entropy
  of an action distribution $p$. The teacher entropy is independent of $\theta$, so SFT
  and forward KL have the same expected gradient. Combining OPD on student
  actions with SFT on teacher actions yields
\begin{equation}
\ell_t^{\mathrm{UOPD}}=
\begin{cases}
\ell_t^{\mathrm{OPD}}(a_t^S;\theta),
& \delta_t\le\tau_n,\\[3pt]
-\beta\log\pi_\theta(a_t^T\mid h_t),
& \delta_t>\tau_n,
\end{cases}
\label{eq:said_step}
\end{equation}
  where $\beta$ controls the weight of the teacher-action SFT loss.
  The intervention decision thus determines both the executed action
  and the KL direction at each turn.

  Appendix~\ref{app:mode_balance} illustrates the
  mode-seeking and mode-covering tendencies of reverse and forward KL,
  respectively.

\subsection{Understanding Intervention and Imitation}
\label{sec:theory}

The intervention mechanism in Section~\ref{subsec:said} gives each teacher correction two roles: an action to execute and a target to imitate. We first characterize how executing selected corrections changes rollout return. We then show how teacher-action imitation can improve the student's decision at a step where reverse-KL supervision would reduce return.

\subsubsection{Improving Rollout Quality}
\label{sec:theory_execution}

We measure rollout quality by the expected episode return $J(\pi)$ and write $\pi_S:=\pi_\theta$ for the current student. During collection, fix $\pi_S$ and the intervention threshold $\tau_n$. The student proposes $a_t^S\sim\pi_S(\cdot\mid h_t)$, and $I_t=\mathbf{1}\{\mathcal U(\pi_T,\pi_S;h_t,a_t^S)>\tau_n\}$ determines whether a teacher action $a_t^T\sim\pi_T(\cdot\mid h_t)$ replaces it. Let $\pi_G$ denote this assisted execution policy and $N_{\mathrm{int}}:=\sum_t I_t$. The student and assisted policies interact with the same environment and initial task distribution. Write $Q_t^{\pi_S}(h,a)$ for the expected remaining return from taking $a$ at $(t,h)$ and then following the student. Expectations under $\pi_G$ include the student proposals as well as the executed actions.

\begin{assumption}[Average recoverability]
\label{ass:average_recoverability}
Whenever $\mathbb E_{\pi_G}[N_{\mathrm{int}}]>0$, teacher actions have average advantage of at least $\gamma>0$ at the selected turns:
\begin{equation}
\frac{
\mathbb E_{\pi_G}\!\left[\sum_t I_t
\bigl(Q_t^{\pi_S}(h_t,a_t^T)-Q_t^{\pi_S}(h_t,a_t^S)\bigr)\right]
}{\mathbb E_{\pi_G}[N_{\mathrm{int}}]}
\ge \gamma.
\label{eq:average_recoverability}
\end{equation}
\end{assumption}

This condition averages over the selected turns and the histories reached by $\pi_G$; it does not require every teacher correction to improve return.

\begin{theorem}[Return improvement under selective intervention]
\label{thm:return_decomposition}
For any uncertainty intervention rule and teacher policy in a finite-horizon environment,
\begin{equation}
J(\pi_G)-J(\pi_S)
=\mathbb E_{\pi_G}\!\left[
\sum_t I_t
\bigl(Q_t^{\pi_S}(h_t,a_t^T)-Q_t^{\pi_S}(h_t,a_t^S)\bigr)
\right].
\label{eq:return_decomposition}
\end{equation}
Under Assumption~\ref{ass:average_recoverability}, this identity gives
\begin{equation}
J(\pi_G)-J(\pi_S)
\ge \gamma\,\mathbb E_{\pi_G}[N_{\mathrm{int}}]\ge0.
\label{eq:return_lower_bound}
\end{equation}
\end{theorem}

Theorem~\ref{thm:return_decomposition} relates rollout improvement to the downstream value of the selected corrections. Each term compares a teacher action with the student proposal it replaces, while the expectation accounts for the histories reached after earlier interventions. Under average recoverability, the gain is at least $\gamma$ per intervention in expectation. This is the execution benefit illustrated by Fig.~\ref{fig:motivation}(c): an action correction improves the subsequent rollout before any student update. The proof is in Appendix~\ref{app:proof_return_improvement}.

\subsubsection{Learning from Teacher Corrections}
\label{sec:theory_supervision}

Beyond changing the rollout, intervention supplies a teacher action as the learning target at the current turn. We examine how this changes the student's update at a decision step. Reverse KL can penalize an overproduced action without directing the resulting probability change toward higher-return actions. We show that teacher-action SFT can improve the update direction even when the student can represent the teacher exactly. 
At a fixed turn $t$ and history $h$, let $J_{t,h}(\pi):=\mathbb E_{a\sim\pi(\cdot\mid h)}[Q_t^{\pi_S}(h,a)]$ measure the return of an action distribution, keeping $\pi_S$ as the continuation policy.

\begin{proposition}[Benefit of teacher-action imitation]
\label{prop:corrective_supervision}
There exist a one-step, three-action decision problem with binary terminal rewards, full-support student and teacher policies with $J_{t,h}(\pi_T)>J_{t,h}(\pi_S)$, and an unrestricted softmax student $\pi_\theta(\cdot\mid h)=\operatorname{softmax}(\theta)$ initialized at $\pi_S$. Let $\pi_{\mathrm{OPD}}^+$ and $\pi_{\mathrm{SFT}}^+$ result from one ordinary gradient step on the action logits, starting from the same student and using the same step size $\alpha$, with the population gradients of reverse KL and teacher-action SFT at $h$, respectively. For all sufficiently small $\alpha>0$,
\begin{equation}
J_{t,h}(\pi_{\mathrm{OPD}}^+)
<J_{t,h}(\pi_S)
<J_{t,h}(\pi_{\mathrm{SFT}}^+).
\label{eq:corrective_supervision_separation}
\end{equation}
\end{proposition}

\begin{table}[!t]
\centering
\caption{\textbf{Main results on ALFWorld and WebShop.} We report success rate (SR, \%) on ALFWorld, and mean task score, SR (\%) on WebShop. We also include average number of turns. Results are means $\pm$ standard deviations over three evaluation seeds. Bold indicates the best result among distillation methods.}
\label{tab:exp-agentic}
\small
\setlength{\tabcolsep}{3pt}
\renewcommand{\arraystretch}{1.12}
\begin{tabular*}{\textwidth}{@{\extracolsep{\fill}}l cccc ccc@{}}
\toprule
& \multicolumn{4}{c}{\textbf{ALFWorld}} & \multicolumn{3}{c}{\textbf{WebShop}} \\
\cmidrule(lr){2-5}\cmidrule(lr){6-8}
& \multicolumn{2}{c}{Seen} & \multicolumn{2}{c}{Unseen} & \multicolumn{3}{c}{Test} \\
\cmidrule(lr){2-3}\cmidrule(lr){4-5}\cmidrule(lr){6-8}
Method & SR $\uparrow$ & Turns $\downarrow$ & SR $\uparrow$ & Turns $\downarrow$ & Score $\uparrow$ & SR $\uparrow$ & Turns $\downarrow$ \\
\midrule
\multicolumn{8}{c}{\textit{RL-Qwen2.5-7B teacher $\rightarrow$ Qwen2.5-3B student}} \\
\addlinespace[2pt]
Student (zero-shot) & 21.0\,\scriptsize{$\pm$0.8} & 43.8\,\scriptsize{$\pm$0.3} & 15.9\,\scriptsize{$\pm$3.4} & 46.0\,\scriptsize{$\pm$1.1} & 12.2\,\scriptsize{$\pm$2.1} & 1.0\,\scriptsize{$\pm$0.5} & 13.6\,\scriptsize{$\pm$0.3} \\
Teacher (zero-shot) & 97.4\,\scriptsize{$\pm$0.4} & 9.1\,\scriptsize{$\pm$0.1} & 92.8\,\scriptsize{$\pm$0.4} & 12.3\,\scriptsize{$\pm$0.4} & 85.4\,\scriptsize{$\pm$0.6} & 77.9\,\scriptsize{$\pm$1.2} & 6.6\,\scriptsize{$\pm$0.3} \\
\midrule
OPD & 88.1\,\scriptsize{$\pm$1.8} & 14.2\,\scriptsize{$\pm$0.3} & 85.3\,\scriptsize{$\pm$1.6} & 16.3\,\scriptsize{$\pm$0.6} & 77.5\,\scriptsize{$\pm$2.2} & 66.9\,\scriptsize{$\pm$3.0} & 7.4\,\scriptsize{$\pm$0.3} \\
TCOD-F2B & 85.7\,\scriptsize{$\pm$1.2} & 14.2\,\scriptsize{$\pm$0.2} & 87.1\,\scriptsize{$\pm$0.9} & 15.6\,\scriptsize{$\pm$0.7} & 76.4\,\scriptsize{$\pm$1.8} & 64.3\,\scriptsize{$\pm$1.2} & 7.4\,\scriptsize{$\pm$0.2} \\
TCOD-B2F & 86.4\,\scriptsize{$\pm$2.1} & 14.1\,\scriptsize{$\pm$0.6} & 85.6\,\scriptsize{$\pm$1.1} & 15.9\,\scriptsize{$\pm$0.4} & 76.3\,\scriptsize{$\pm$0.5} & 65.1\,\scriptsize{$\pm$2.4} & 7.5\,\scriptsize{$\pm$0.1} \\
FTB-OPD & 87.4\,\scriptsize{$\pm$1.1} & 14.1\,\scriptsize{$\pm$0.3} & 86.6\,\scriptsize{$\pm$2.2} & \textbf{14.7}\,\scriptsize{$\pm$0.7} & 77.3\,\scriptsize{$\pm$0.3} & 69.0\,\scriptsize{$\pm$1.6} & 7.4\,\scriptsize{$\pm$0.1} \\
\textbf{UOPD} & \textbf{90.0}\,\scriptsize{$\pm$0.7} & \textbf{13.6}\,\scriptsize{$\pm$0.3} & \textbf{88.8}\,\scriptsize{$\pm$1.3} & 14.9\,\scriptsize{$\pm$0.8} & \textbf{82.4}\,\scriptsize{$\pm$2.0} & \textbf{69.3}\,\scriptsize{$\pm$2.0} & \textbf{6.5}\,\scriptsize{$\pm$0.1} \\
\midrule
\multicolumn{8}{c}{\textit{RL-Qwen2.5-7B teacher $\rightarrow$ Qwen2.5-1.5B student}} \\
\addlinespace[2pt]
Student (zero-shot) & 6.2\,\scriptsize{$\pm$0.4} & 48.5\,\scriptsize{$\pm$0.3} & 3.7\,\scriptsize{$\pm$0.7} & 49.2\,\scriptsize{$\pm$0.2} & 22.6\,\scriptsize{$\pm$3.8} & 2.3\,\scriptsize{$\pm$2.1} & 10.9\,\scriptsize{$\pm$0.4} \\
Teacher (zero-shot) & 97.4\,\scriptsize{$\pm$0.4} & 9.1\,\scriptsize{$\pm$0.1} & 92.8\,\scriptsize{$\pm$0.4} & 12.3\,\scriptsize{$\pm$0.4} & 85.4\,\scriptsize{$\pm$0.6} & 77.9\,\scriptsize{$\pm$1.2} & 6.6\,\scriptsize{$\pm$0.3} \\
\midrule
OPD & 86.2\,\scriptsize{$\pm$2.9} & 15.2\,\scriptsize{$\pm$1.0} & 84.8\,\scriptsize{$\pm$0.4} & 17.2\,\scriptsize{$\pm$0.9} & 66.3\,\scriptsize{$\pm$0.9} & 52.3\,\scriptsize{$\pm$1.4} & 7.9\,\scriptsize{$\pm$0.4} \\
TCOD-F2B & 86.0\,\scriptsize{$\pm$1.8} & 14.3\,\scriptsize{$\pm$0.8} & 85.1\,\scriptsize{$\pm$2.0} & 16.8\,\scriptsize{$\pm$1.0} & 69.5\,\scriptsize{$\pm$1.2} & 51.8\,\scriptsize{$\pm$2.5} & 7.2\,\scriptsize{$\pm$0.2} \\
TCOD-B2F & 84.5\,\scriptsize{$\pm$0.4} & 15.2\,\scriptsize{$\pm$0.1} & 84.1\,\scriptsize{$\pm$0.9} & 18.0\,\scriptsize{$\pm$0.1} & 68.4\,\scriptsize{$\pm$1.2} & 54.2\,\scriptsize{$\pm$0.5} & 7.6\,\scriptsize{$\pm$0.1} \\
FTB-OPD & 85.2\,\scriptsize{$\pm$1.1} & 15.5\,\scriptsize{$\pm$0.4} & 83.6\,\scriptsize{$\pm$1.3} & 17.7\,\scriptsize{$\pm$0.4} & 67.7\,\scriptsize{$\pm$0.5} & 55.2\,\scriptsize{$\pm$1.2} & 9.4\,\scriptsize{$\pm$0.1} \\
\textbf{UOPD} & \textbf{89.8}\,\scriptsize{$\pm$0.8} & \textbf{13.0}\,\scriptsize{$\pm$0.5} & \textbf{86.6}\,\scriptsize{$\pm$0.7} & \textbf{15.3}\,\scriptsize{$\pm$0.2} & \textbf{76.8}\,\scriptsize{$\pm$1.8} & \textbf{58.6}\,\scriptsize{$\pm$0.0} & \textbf{6.8}\,\scriptsize{$\pm$0.4} \\
\bottomrule
\end{tabular*}
\\[3pt]
\vspace{-5mm}
\end{table}

Proposition~\ref{prop:corrective_supervision} identifies a decision step where the supervision objective determines whether an update improves return. In the construction, reverse KL reduces an overproduced successful action but transfers some of its probability to an unsuccessful action. Teacher-action SFT instead reduces the unsuccessful action's probability. This motivates a teacher imitation target at intervention turns within UOPD, which retains OPD supervision on student-executed actions. The parameterized construction and full proof are in Appendix~\ref{app:proof_corrective_supervision}.

\section{Experiments}
\label{sec:experiments}

\subsection{Experimental Setup}
\label{subsec:exp-setup}

\noindent\textbf{Benchmarks and metrics.} We evaluate on three multi-turn agent benchmarks: ALFWorld \citep{shridhar2021alfworld},
WebShop \citep{yao2022webshop}, and Search \citep{jin2025searchr1}. These benchmarks cover embodied household tasks, web shopping, and search-based question answering, allowing us to assess the effectiveness and efficiency of our uncertainty-guided intervention across different environments. We report success rate (SR) on both the seen and unseen splits of ALFWorld, task score and success rate on WebShop, and exact match on four multi-hop Search datasets. Average interaction turns measure efficiency.
Dataset and evaluation details are provided in Appendix~\ref{app:provenance}.

\noindent\textbf{Baselines.} We compare UOPD with four baselines.
\textbf{OPD} \citep{agarwal2024policy,lu2025onpolicydistillation} applies reverse-KL supervision to student-generated actions throughout a rollout.
\textbf{TCOD-F2B} \citep{wang2026exploring} progressively extends the student's rollout horizon from the initial state.
\textbf{TCOD-B2F} \citep{wang2026exploring} starts student rollouts after expert action prefixes and gradually shortens these prefixes to expose the student to earlier decisions.
\textbf{FTB-OPD} \citep{futurebridgeopd2026} introduces teacher corrections at high-disagreement steps and retains them for distillation when they improve teacher preference over subsequent student continuations.

\noindent\textbf{Training details.} On ALFWorld and WebShop, we use Qwen2.5-3B-Instruct and Qwen2.5-1.5B-Instruct \citep{yang2024qwen25} as
students, with task-specific GiGPO-Qwen2.5-7B-Instruct teachers trained
using reinforcement learning  \citep{feng2025gigpo}.
For Search, we distil a base Qwen2.5-1.5B student from
SearchR1-Qwen2.5-7B-em-ppo \citep{jin2025searchr1}. For all methods, we train for $250$ steps on ALFWorld and $150$ on WebShop with
rollout batch size $16$, and for $300$ steps on Search with batch size $64$. For ALFWorld and WebShop, UOPD uses teacher confidence as its uncertainty signal. More training configurations and implementation details are provided in Appendix~\ref{app:provenance}.

\subsection{Main Results}
\label{subsec:exp-alfworld}
\label{subsec:exp-main}

We evaluate UOPD’s task performance and efficiency on ALFWorld and WebShop (Table~\ref{tab:exp-agentic})
  and multi-hop question answering (Fig.~\ref{fig:search-15b-multihop}).

\noindent\textbf{UOPD improves task performance and interaction efficiency.}
Across both student sizes, UOPD achieves the highest success rates on the ALFWorld seen and
unseen splits and the highest WebShop scores among the compared distillation methods.
The gains on unseen ALFWorld tasks show that the benefit extends to environments beyond those encountered during training,
while the improvements at $1.5$B demonstrate effective transfer across a larger
teacher--student capacity gap. On WebShop, UOPD raises the $1.5$B student's
score from OPD's $66.3$ to $76.8$ and success rate from $52.3\%$ to $58.6\%$, while reducing
average turns from $7.9$ to $6.8$. UOPD requires the fewest interaction turns in five
of the six evaluation settings, demonstrating its efficiency across tasks and student sizes.

\noindent\begin{minipage}[t]{0.49\textwidth}
\vspace{0pt}
\textbf{The gains extend to multi-hop search.}
Multi-hop search requires the model to retrieve
  relevant evidence and combine it across multiple reasoning
  steps. With a $1.5$B student, UOPD consistently outperforms OPD
  on all four datasets, improving performance by $2.67$–$8.00$
  percentage points, demonstrating the
  benefit of uncertainty-guided supervision for search tasks. Meanwhile, the average number of turns decreases from $5.19$ to $3.90$, indicating
  that improved accuracy is accompanied by higher efficiency at inference time.
\end{minipage}\hfill
\begin{minipage}[t]{0.48\textwidth}
\vspace{0pt}
\centering
\captionsetup{type=figure,skip=3pt}
\includegraphics[width=\linewidth]{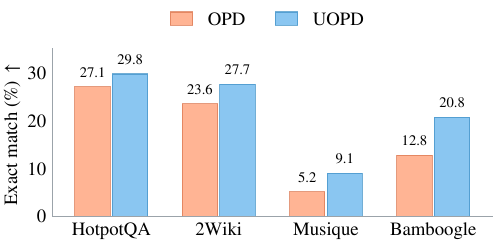}
\caption{\textbf{Multi-hop Search results.}}
\label{fig:search-15b-multihop}
\end{minipage}
\par

\vspace{-2.5mm}
\begin{figure}[!htb]
\centering
\captionsetup{skip=3pt}
\includegraphics[width=\textwidth]{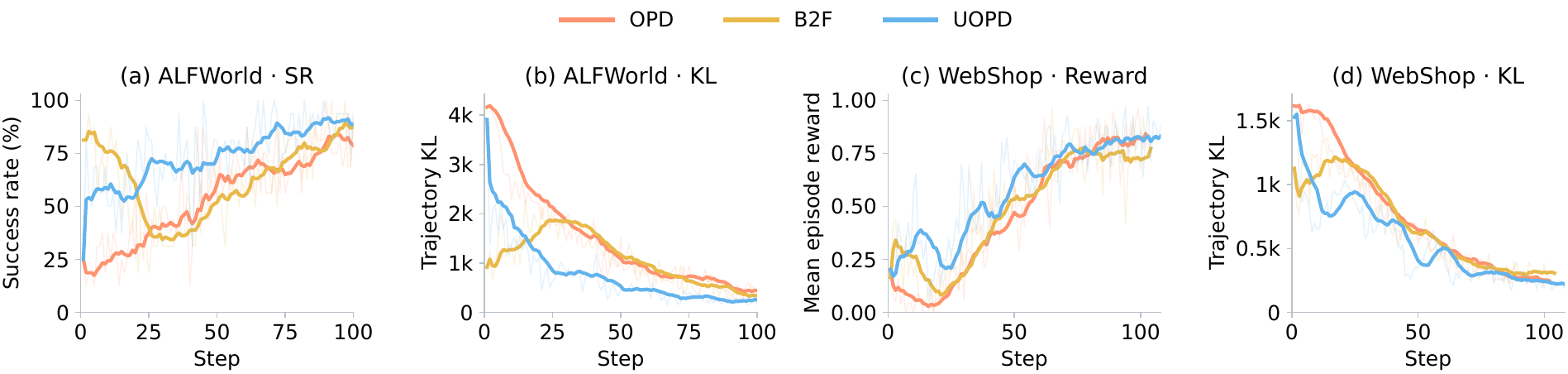}
\caption{\textbf{Training dynamics with 3B students.} We compare OPD, TCOD-B2F (B2F),
and UOPD on ALFWorld (a,b: success rate and teacher--student KL) and WebShop
(c,d: reward and teacher--student KL).}
\label{fig:training-guidance}
\end{figure}

\noindent\textbf{Effective guidance during training.}
  Fig.~\ref{fig:training-guidance} shows that UOPD improves
  rollout quality early in training, achieving higher success
  rates on ALFWorld and higher rewards on WebShop than OPD.
  UOPD also exhibits a faster decline in trajectory-level
  teacher--student KL than OPD and B2F. B2F improves initial
  rollout quality through teacher-executed prefixes, but its
  performance drops as these prefixes shorten. This suggests
  that learning from teacher-provided states may not prepare
  the student to execute longer action sequences independently.
  UOPD targets high-uncertainty
  turns throughout the rollout, providing teacher corrections
  while letting the student act and learn throughout the task.

\subsection{Additional Analysis}
\label{subsec:exp-ablation}

Table~\ref{tab:exp-ablation} presents ablation studies with $3$B students on ALFWorld and WebShop,
examining how individual components of UOPD contribute to task performance and interaction efficiency.

\begin{table}[!ht]
\centering
\caption{\textbf{Effects of intervention selection, SFT weight, teacher execution, and schedule
with a 3B student.} The default UOPD configuration uses teacher confidence,
$\beta=1$, teacher execution, and a $30\%\!\to\!5\%$ intervention schedule. Random intervention
samples each turn with the scheduled probability. Values are means over three evaluation seeds, and bold marks the best value
in each column.}
\label{tab:exp-ablation}
\footnotesize
\setlength{\tabcolsep}{4pt}
\begin{tabular*}{\textwidth}{@{\extracolsep{\fill}}l ccccc@{}}
\toprule
& \multicolumn{2}{c}{ALFWorld} & \multicolumn{3}{c}{WebShop} \\
\cmidrule(lr){2-3}\cmidrule(lr){4-6}
Configuration & Seen SR $\uparrow$ & Unseen SR $\uparrow$ & Score $\uparrow$ & SR $\uparrow$ & Turns $\downarrow$ \\
\midrule
\textbf{UOPD} ($\beta=1$, $30\%\!\to\!5\%$) & 90.0 & 88.8 & \textbf{82.4} & \textbf{69.3} & \textbf{6.5} \\
\midrule
\multicolumn{6}{@{}l}{\emph{Intervention signal}} \\
\quad Confidence gap & 89.8 & 87.8 & 77.3 & 69.0 & 7.2 \\
\quad Random intervention & 88.3 & 83.8 & 77.3 & 64.6 & 7.4 \\
\midrule
\multicolumn{6}{@{}l}{\emph{SFT weight}} \\
\quad $\beta=0$ & 86.2 & 80.8 & 79.7 & 66.1 & 6.7 \\
\quad $\beta=0.5$ & 85.7 & 86.3 & 78.0 & 68.0 & 7.3 \\
\quad $\beta=2$ & 89.5 & 86.8 & 77.3 & 68.0 & 7.5 \\
\midrule
\multicolumn{6}{@{}l}{\emph{Teacher execution}} \\
\quad w/o teacher execution & 89.5 & 84.3 & 77.6 & 69.0 & 7.2 \\
\midrule
\multicolumn{6}{@{}l}{\emph{Intervention schedule}} \\
\quad Constant $15\%$ & 90.0 & \textbf{89.6} & 78.0 & 67.7 & 7.0 \\
\quad $50\%\!\to\!30\%$ & \textbf{90.5} & 86.6 & 69.9 & 59.4 & 8.2 \\
\bottomrule
\end{tabular*}
\vspace{-4mm}
\end{table}

\noindent\textbf{Selecting intervention turns.}
We replace uncertainty-based selection with random intervention, sampling each
turn with the same target-rate schedule while retaining teacher execution and
the SFT loss ($\beta=1$). Compared with random intervention, UOPD improves
unseen ALFWorld success from $83.8\%$ to $88.8\%$ and WebShop score from
$77.3$ to $82.4$, while reducing WebShop turns from $7.4$ to $6.5$.
Using the length-normalized student--teacher log-probability gap as the
intervention signal also yields higher success rates than random selection
on both environments. Teacher confidence achieves comparable success rates
while providing higher task scores and shorter interactions on WebShop.

\noindent\textbf{Learning from teacher corrections.}
Removing the SFT loss reduces performance on both benchmarks, with the largest
success-rate drop of $8.0$ percentage points on unseen ALFWorld. Teacher execution
alone therefore does not recover the full benefit of learning from corrections.
The default $\beta=1$ performs best among the tested weights; increasing it to
$\beta=2$ brings no further gains.

\noindent\textbf{Executing teacher corrections.}
On WebShop, teacher execution increases the score from $77.6$ to $82.4$ and
reduces average turns from $7.2$ to $6.5$, while success rates remain similar.
The benefit thus extends beyond whether a task is fully completed: the student
better satisfies the shopping requirements and uses fewer steps per episode.
This suggests that executing corrections during training can complement the
learning signal provided by the SFT loss.

\noindent\textbf{Allocating teacher intervention.}
A higher intervention rate does not consistently improve performance.
The $50\%\!\to\!30\%$ schedule slightly improves seen ALFWorld success but
lowers unseen success and performs worse on WebShop. A constant $15\%$ rate
remains competitive on ALFWorld, achieving the highest unseen success among
the tested schedules, whereas $30\%\!\to\!5\%$ yields the highest score and
fewest turns on WebShop.
See Appendix~\ref{app:telemetry} for intervention dynamics.

\noindent\begin{minipage}[t]{0.49\textwidth}
\vspace{0pt}
\textbf{Training efficiency.}
Fig.~\ref{fig:training-time-3b} compares training time at equal step counts per benchmark. UOPD reduces
training time relative to FTB-OPD by $13.8\%$ on ALFWorld and $36.5\%$ on WebShop,
with $11.5\%$ and $14.7\%$ overhead over OPD. UOPD selects corrections before action
execution, avoiding the paired student continuations used by FTB-OPD to validate
teacher guidance. It thus improves performance with modest training cost compared with standard on-policy distillation.
\end{minipage}\hfill
\begin{minipage}[t]{0.48\textwidth}
\vspace{0pt}
\centering
\captionsetup{type=figure,skip=3pt}
\includegraphics[width=\linewidth]{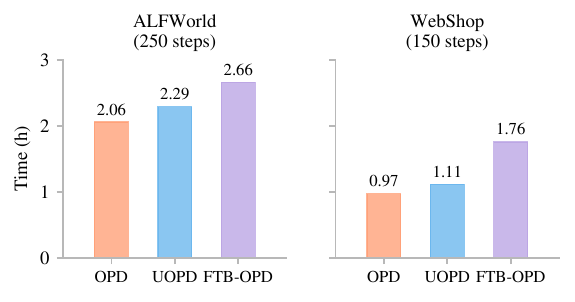}
\caption{\textbf{Training time with 3B students.}}
\label{fig:training-time-3b}
\end{minipage}
\par

\FloatBarrier

\section{Conclusion}
\label{sec:conclusion}

In this work, we investigate when teacher intervention improves the trajectories used for
on-policy distillation of multi-turn agents. Our controlled experiments show that a teacher
correction at a low-confidence step improves subsequent student behavior and task success,
outperforming random intervention. Building on this
finding, we propose UOPD, which uses teacher uncertainty to select turns for action correction
and imitation learning while retaining standard OPD on student actions. Experiments on
ALFWorld, WebShop, and multi-hop Search demonstrate that UOPD improves task
performance over OPD while reducing interaction turns.

Our findings highlight the value of teacher uncertainty for guiding both what an agent learns
and the trajectories from which it learns. Selective corrections provide useful learning targets
and steer subsequent interaction toward better outcomes. Future work could extend this approach
to longer-horizon tasks and adapt the intervention budget to task difficulty and student progress.

\clearpage
\bibliography{references}

\appendix
\newpage
\section{Proofs}
\label{app:theory_proofs}

\subsection{Proof of Theorem~\ref{thm:return_decomposition}}
\label{app:proof_return_improvement}

\begin{proof}
Define the student's finite-horizon value and advantage functions by
\begin{align*}
V_t^{\pi_S}(h)
&:=\mathbb E_{a\sim\pi_S(\cdot\mid h)}Q_t^{\pi_S}(h,a),\\
A_t^{\pi_S}(h,a)
&:=Q_t^{\pi_S}(h,a)-V_t^{\pi_S}(h),
\end{align*}
with $V_H^{\pi_S}=0$.

\paragraph{Step 1: derive the finite-horizon performance-difference identity.}
By the Bellman definition of $Q_t^{\pi_S}$,
\[
Q_t^{\pi_S}(h_t,a_t)
=\mathbb E\!\left[r_t+V_{t+1}^{\pi_S}(h_{t+1})
\mid h_t,a_t\right].
\]
Taking expectation under the assisted trajectory and summing over turns gives
\begin{align*}
\sum_{t=0}^{H-1}\mathbb E_{\pi_G}\!\left[A_t^{\pi_S}(h_t,a_t)\right]
&=\sum_{t=0}^{H-1}\mathbb E_{\pi_G}\!\left[
r_t+V_{t+1}^{\pi_S}(h_{t+1})-V_t^{\pi_S}(h_t)\right]\\
&=J(\pi_G)-J(\pi_S).
\end{align*}
The value terms telescope because the policies share the distribution of initial task inputs and $V_H^{\pi_S}=0$.

\paragraph{Step 2: evaluate the advantage at an intervention turn.}
Condition on $(t,h_t)$ and sample the student proposal and teacher action as in the proposal-coupled construction. Since $a_t=a_t^S$ when $I_t=0$ and $a_t=a_t^T$ when $I_t=1$,
\begin{align*}
&\mathbb E\!\left[A_t^{\pi_S}(h_t,a_t)\mid h_t\right]\\
&=\mathbb E\!\left[
I_tA_t^{\pi_S}(h_t,a_t^T)+(1-I_t)A_t^{\pi_S}(h_t,a_t^S)
\,\middle|\,h_t\right]\\
&=\mathbb E\!\left[
I_t\bigl(A_t^{\pi_S}(h_t,a_t^T)-A_t^{\pi_S}(h_t,a_t^S)\bigr)
\,\middle|\,h_t\right]\\
&=\mathbb E\!\left[
I_t\bigl(Q_t^{\pi_S}(h_t,a_t^T)-Q_t^{\pi_S}(h_t,a_t^S)\bigr)
\,\middle|\,h_t\right].
\end{align*}
The third line uses $\mathbb E[A_t^{\pi_S}(h_t,a_t^S)\mid h_t]=0$; the final line cancels the common value term.

\paragraph{Step 3: aggregate over the assisted trajectory.}
Substituting Step 2 into Step 1 and applying the law of total expectation gives Eq.~\ref{eq:return_decomposition}.

Under Assumption~\ref{ass:average_recoverability}, multiplying
Eq.~\ref{eq:average_recoverability} by
$\mathbb E_{\pi_G}[N_{\mathrm{int}}]$ and applying the theorem yields
\[
J(\pi_G)-J(\pi_S)
\ge\gamma\,\mathbb E_{\pi_G}[N_{\mathrm{int}}]\ge0,
\]
which proves the lower bound in Theorem~\ref{thm:return_decomposition}. If the expected number of
interventions is zero, $\pi_G=\pi_S$ almost surely and the same conclusion
holds directly.
\end{proof}

\subsection{Proof of Proposition~\ref{prop:corrective_supervision}}
\label{app:proof_corrective_supervision}

\begin{proof}
Fix a history $h$ at a terminal decision turn $t$. There are three actions $\{a,b,c\}$ with rewards $r(a)=r(b)=1$ and $r(c)=0$, so $Q_t^{\pi_S}(h,i)=r(i)$. Suppress $h$ in the distributions and use the unrestricted softmax parameterization
\[
p_i(\theta)=\frac{\exp(\theta_i)}{\sum_{j\in\{a,b,c\}}\exp(\theta_j)}.
\]
Initialize $\theta_0=(0,0,0)$, so $p=\pi_S=(1/3,1/3,1/3)$. The three logits are independent parameters. Choose a teacher
\[
u=\pi_T=\bigl(v(1-\epsilon),v\epsilon,1-v\bigr),
\qquad
\frac23<v<1,\quad 0<\epsilon<\frac12,
\]
with
\begin{equation}
v^2\epsilon(1-\epsilon)<(1-v)^2.
\label{eq:corrective_teacher_family}
\end{equation}
For any $v\in(2/3,1)$, this condition holds for sufficiently small positive $\epsilon$. Thus the construction describes a family of full-support teachers, each with return $v>2/3=J_{t,h}(\pi_S)$. Every teacher in this family is represented by the student class, for example by logits $\theta_i=\log u_i$.

Write $\ell_i=\log(p_i/u_i)$ and $\bar\ell=\sum_i p_i\ell_i$. The ordinary logit gradients of the two losses in Proposition~\ref{prop:corrective_supervision} are
\begin{equation}
G_{\mathrm{OPD},i}=p_i(\ell_i-\bar\ell),
\qquad
G_{\mathrm{SFT},i}=p_i-u_i.
\label{eq:corrective_gradients}
\end{equation}
Both gradients sum to zero. The first follows by differentiating $\sum_i p_i\log(p_i/u_i)$ and using $\nabla_\theta\log p_i=e_i-p$; the second is the expected teacher-action SFT gradient. For $M\in\{\mathrm{OPD},\mathrm{SFT}\}$, the updated logits are $\theta_M^+=\theta_0-\alpha G_M$, and $\pi_M^+=\operatorname{softmax}(\theta_M^+)$.

The softmax Jacobian is $F=\operatorname{diag}(p)-pp^\top$. At the uniform initial policy, $F=\frac13 I-\frac19\mathbf1\mathbf1^\top$, hence the first-order probability change under update $M$ is
\[
\left.\frac{d}{d\alpha}\pi_M^+\right|_{\alpha=0}
=-FG_M=-\frac13G_M.
\]
Since the reward vector is $(1,1,0)$ and $\sum_iG_{M,i}=0$, the return derivative equals $G_{M,c}/3$. Therefore
\begin{align}
\left.\frac{d}{d\alpha}J_{t,h}(\pi_{\mathrm{OPD}}^+)\right|_{\alpha=0}
&=\frac1{27}\log\frac{u_a u_b}{u_c^2}
=\frac1{27}\log\frac{v^2\epsilon(1-\epsilon)}{(1-v)^2}<0,\notag\\
\left.\frac{d}{d\alpha}J_{t,h}(\pi_{\mathrm{SFT}}^+)\right|_{\alpha=0}
&=\frac{v-2/3}{3}>0.
\label{eq:corrective_return_derivatives}
\end{align}

Let $d_{\mathrm{OPD}}$ be the negative of the first derivative and $d_{\mathrm{SFT}}$ the second derivative in Eq.~\ref{eq:corrective_return_derivatives}. Both are strictly positive. Smoothness gives
\begin{align*}
J_{t,h}(\pi_{\mathrm{OPD}}^+)
&=J_{t,h}(\pi_S)-d_{\mathrm{OPD}}\alpha+O(\alpha^2),\\
J_{t,h}(\pi_{\mathrm{SFT}}^+)
&=J_{t,h}(\pi_S)+d_{\mathrm{SFT}}\alpha+O(\alpha^2).
\end{align*}
Choose $\alpha_0>0$ small enough that each remainder has magnitude at most $d_M\alpha/2$ for every $0<\alpha\le\alpha_0$. The OPD update then strictly decreases return, while the SFT update strictly increases it, proving Eq.~\ref{eq:corrective_supervision_separation}.

All probabilities are positive, and the teacher advantage and derivative inequalities are strict. By continuity, these inequalities persist under sufficiently small perturbations of the initial logits and teacher probabilities. The comparison concerns the update direction: both population losses admit the representable teacher as a global minimizer.
\end{proof}

\section{Mode Seeking and Mode Covering}
\label{app:mode_balance}

Figure~\ref{fig:mode_balance} illustrates the different optimization tendencies of the two KL directions at a fixed interaction history. Reverse KL emphasizes actions already sampled by the student and favors concentrating probability on teacher-supported modes. Forward KL penalizes low student probability on actions supported by the teacher, encouraging coverage of behaviors the student rarely generates. UOPD applies reverse KL to retained student actions and implements forward KL through SFT on teacher corrections. Executing these corrections also changes subsequent interaction histories; this effect on the rollout comes from action replacement, while SFT provides the learning signal.

\begin{figure}[htbp]
\centering
\includegraphics[width=\linewidth]{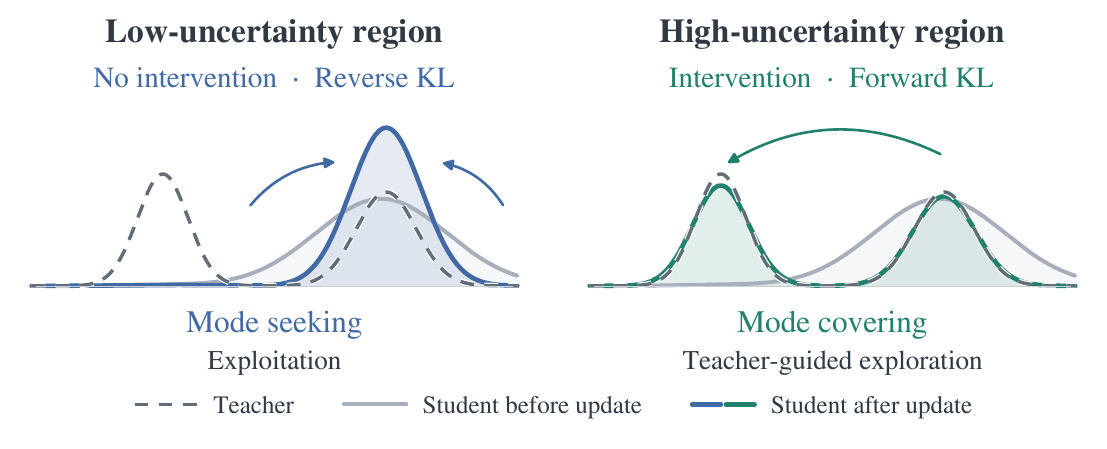}
\caption{\textbf{Mode seeking and mode covering in UOPD.}
Left: reverse KL refines student probability around a teacher-supported mode.
Right: forward KL encourages the student to cover teacher behaviors it rarely generates.
The curves schematically illustrate the two loss directions rather than measured policy distributions.}
\label{fig:mode_balance}
\end{figure}
\FloatBarrier

\section{Details of the ALFWorld Motivation Study}
\label{app:motivation_details}

This section provides the protocol for the controlled study in
Section~\ref{sec:motivation}. We compare student-only rollouts, one teacher
intervention at a low-confidence step, and one teacher intervention at a random
step. Both models remain fixed, so the comparison measures how an action
correction changes subsequent student behavior without a parameter update.

\subsection{Models, Tasks, and Rollout Collection}

We collect student-only rollouts on 100 tasks from the ALFWorld valid-seen split
\citep{shridhar2021alfworld}. The student is Qwen2.5-3B-Instruct, evaluated
zero-shot, and the teacher is GiGPO-Qwen2.5-7B-Instruct-ALFWorld
\citep{feng2025gigpo}. Episodes terminate on task completion or after 50
environment turns. We use sampling temperature $0.4$, a maximum of 512 response
tokens per turn, and seed 42. Each response contains reasoning and an environment
action; the resulting observation is appended to the interaction history.
We save the exact prompt and response token IDs and score the recorded student
responses under the teacher after rollout collection.

\subsection{Teacher Confidence and Intervention Selection}

For a student response $a_t^S=(y_{t,1},\ldots,y_{t,L_t})$ at turn $t$, we define
teacher confidence as the length-normalized log-probability
\begin{equation}
C_t := \frac{1}{L_t}\sum_{j=1}^{L_t}
\log\pi_T(y_{t,j}\mid h_t,y_{t,<j}),
\label{eq:app_teacher_confidence}
\end{equation}
where $h_t$ is the interaction history defined in Section~\ref{subsec:preliminary}. We score all response
tokens, including reasoning and the structured action, and exclude prompt
tokens. Scores are computed from unscaled teacher logits and reported in nats
per token; higher values indicate greater teacher support for the student
response. The corresponding uncertainty score is $\delta_t=-C_t$.

We set $\tau_C$ to the 25th percentile of valid turn-level confidence scores
across the 100 baseline rollouts, giving $\tau_C\approx-1.216$. For trajectory
$i$ with $H_i$ recorded turns, the selected intervention turn is
\begin{equation}
t_i^* := \min\{t:1\le t<H_i-1,\ C_{i,t}<\tau_C\}.
\label{eq:app_motivation_trigger}
\end{equation}
Turn indices are zero-based. This selects the first eligible threshold crossing,
rather than the single lowest-confidence action. We exclude the first and final
baseline turns to retain a replay prefix and subsequent student actions for
comparison. Of the 100 baseline trajectories, 96 meet this criterion and 93 have
verified matching intervention replays. We use these same 93 tasks for all
three comparison conditions.

\subsection{Matched Single-Turn Interventions}

The \textbf{Student} condition uses the original student-only trajectory.
The \textbf{Intervention} condition replaces the action at $t_i^*$ with one
teacher-generated action. The \textbf{Random} condition instead samples one
turn uniformly from $\{1,\ldots,H_i-2\}$ on the same baseline trajectory,
using a task-specific random seed derived from seed 42.

For either intervention condition, we reset the environment to the recorded game
file and execute the saved student actions up to the chosen turn. The teacher
then generates and executes one action, and the student resumes until task
completion or the original 50-turn limit. The intervened rollout therefore
shares the baseline prefix and receives exactly one turn of teacher assistance.
We match trajectories by game file and reconstructed model inputs, and verify
replays against the saved prompts, actions, and final rewards. When Random and
Intervention choose the same turn, they share the saved teacher branch; this
occurs on four tasks, leaving 275 distinct trajectories across the three
conditions.

\subsection{Measurements in Figure~\ref{fig:motivation}}

\paragraph{Teacher confidence.}
Panel (a) averages $C_t$ over baseline trajectories that reach each displayed
turn, using the full set of 100 student-only rollouts. Panel (b) aligns the
matched Student and Intervention trajectories at $t_i^*$ and averages teacher
confidence on student-generated responses at offsets $1,\ldots,5$. The teacher's
replacement action is excluded from this readout. Each mean uses the trajectories
with an available response at that turn; trajectories that have already ended
are not padded. We plot turns with at least 15 observations and connect the
means without smoothing.

\paragraph{Task success.}
Panel (c) reports the percentage of the 93 matched tasks completed successfully
under each condition, using the environment's terminal success indicator.

\paragraph{State returns.}
Panel (d) reports the percentage of trajectories that leave a state and later
return to it. Let $z_0,\ldots,z_H$ denote the symbolic states before the first
action and after each executed action. We identify a state by its exact set of
raw PDDL facts, including predicate names, argument identities and types, and
simulator bookkeeping facts. Observation text, turn numbers, and interaction
history are excluded from this identity. For each trajectory, we compute
\[
I_{\mathrm{return}}(\tau)
=\mathbf{1}\!\left\{\exists\,t\in\{1,\ldots,H\}:
 z_t\ne z_{t-1}\ \text{and}\ z_t\in\{z_0,\ldots,z_{t-1}\}\right\}.
\]
This statistic is computed within each trajectory and covers the full rollout,
including the teacher action where present. Consecutive unchanged states do not
count as a return, and returning does not require repeating the same action.
State returns provide a task-specific indicator of backtracking; a return can
also accompany useful information gathering.
Table~\ref{tab:motivation_protocol} summarizes task success and state returns for all three conditions.

\begin{table}[ht]
\centering
\small
\begin{tabular}{lrr}
\toprule
Condition & Task success (\%) & Trajectories with state returns (\%) \\
\midrule
Student      & 13.98 ($13/93$) & 93.55 ($87/93$) \\
Random       & 16.13 ($15/93$) & 91.40 ($85/93$) \\
Intervention & 25.81 ($24/93$) & 80.65 ($75/93$) \\
\bottomrule
\end{tabular}
\caption{Outcomes on the same 93 ALFWorld tasks used in
Fig.~\ref{fig:motivation}(c,d). Random and Intervention each execute one teacher
action before returning control to the student.}
\label{tab:motivation_protocol}
\end{table}

\section{Experimental Details}
\label{app:provenance}

This section describes the benchmarks and data, the student and teacher models, the baselines, and
the training and evaluation configurations behind the results in Section~\ref{sec:experiments}.
Appendix~\ref{app:prompts} gives the prompts used for training and evaluation.

\subsection{Benchmark Environments}
\label{app:data}

\paragraph{ALFWorld} \citep{shridhar2021alfworld}. Text-based household tasks of six types (pick,
clean, heat, cool, examine, pick-two) in the TextWorld engine. Training draws from the $3{,}553$
training games. We evaluate on the $140$ seen and $134$ unseen validation tasks. A task succeeds when
it is completed within $50$ environment steps, and an episode's turns are its environment steps, with
a failed episode counting the full $50$.

\paragraph{WebShop} \citep{yao2022webshop}. Instruction-following product search and purchase over the
$1{,}000$-product catalog. Training draws from a pool of $6{,}410$ instructions, and evaluation uses
$128$ test instructions that are disjoint from the pool and shared by every method. The task score is
the environment's attribute-match reward on a $0$--$100$ scale, the success rate is the share of
episodes with a perfect match, and the episode cap is $15$ steps.

\paragraph{Search} \citep{jin2025searchr1}. Open-domain question answering with a retrieval tool, in
the Search-R1 setting. Training draws from $169{,}615$ questions ($79{,}168$ from NQ and $90{,}447$
from HotpotQA). The test set is Search-R1's four multi-hop datasets, $22{,}523$ questions in total:
HotpotQA ($7{,}405$), 2WikiMultiHopQA ($12{,}576$), MuSiQue ($2{,}417$) and Bamboogle ($125$). Retrieval uses Search-R1's e5-base-v2 encoder over its 2018 Wikipedia
corpus with a flat FAISS index, returning the top $3$ passages per query. Answers are scored by exact
match after normalisation against any gold answer, using the last \texttt{<answer>} of the episode.
The turn limit is $8$.

\subsection{Students and Teachers}
\label{app:models}

\begin{center}\small
\begin{tabular}{@{}l l l@{}}
\toprule
Benchmark & Students & Teacher \\
\midrule
ALFWorld & Qwen2.5-3B-Instruct, Qwen2.5-1.5B-Instruct & GiGPO-Qwen2.5-7B-Instruct-ALFWorld \\
WebShop & Qwen2.5-3B-Instruct, Qwen2.5-1.5B-Instruct & GiGPO-Qwen2.5-7B-Instruct-WebShop \\
Search & Qwen2.5-1.5B (base) & SearchR1-qwen2.5-7b-em-ppo \\
\bottomrule
\end{tabular}
\end{center}

\noindent The ALFWorld and WebShop teachers are RL-trained with GiGPO \citep{feng2025gigpo}; the Search
teacher is Search-R1's PPO-trained $7$B model, trained from a base checkpoint on Search-R1's raw
prompt, so the Search student is a base checkpoint prompted the same way. In every pair, teacher and
student share a tokenizer, so the teacher can score the student's own tokens.

\subsection{Baselines}
\label{app:baselines}

\paragraph{Zero-shot student and teacher.} Evaluated without training under the protocol of
Appendix~\ref{app:evaluation-protocol}. They bound the results from below and above.

\paragraph{OPD} \citep{agarwal2024policy,lu2025onpolicydistillation}. The student plays the whole
episode, and the teacher scores every student token under the same prompt. Each token receives the
advantage $A_t = \lambda\,(\log\pi_T(y_t) - \log\pi_{S,\mathrm{old}}(y_t))$ with $\lambda=1$, optimised
with a PPO objective (clip $0.2$, dual clip $3.0$, token-mean aggregation). UOPD uses the same
distillation term and adds its intervention rule and the SFT loss on intervened turns.

\paragraph{TCOD-F2B and TCOD-B2F} \citep{wang2026exploring}. Temporal curricula over the rollout horizon,
advanced every $\eta$ rollout batches $n$. F2B lets the student play only the first
$W_n = \min(1+\lfloor n/\eta\rfloor,\,T_{\max})$ turns of an episode and then ends it. B2F replays a
stored expert prefix of $P_n = \mathrm{clamp}(L-1-\lfloor n/\eta\rfloor,\,0,\,L-1)$ actions, where $L$ is
the expert trajectory's length, and the student plays from that point on. The prefix carries no loss.
Both train the student's turns with the OPD loss and are evaluated on the full horizon. On ALFWorld,
$\eta$ is $6$ for F2B and $5$ for B2F.

\paragraph{FTB-OPD} \citep{futurebridgeopd2026}. Introduces teacher corrections at steps where the
student and teacher disagree most, and keeps a correction for distillation only when it improves the
teacher's preference over the student's subsequent continuations. Validating a correction requires
paired student continuations, which is the source of its extra training time in
Figure~\ref{fig:training-time-3b}.

\subsection{Training Configuration}
\label{app:training-config}

We implement UOPD using Trinity-RFT \citep{pan2025trinityrft}.
We use its asynchronous rollout and training
pipeline, with vLLM for student and teacher inference and the verl backend
with fully sharded data parallelism (FSDP) for student
updates. Table~\ref{tab:app-gpu-allocation} gives the eight-GPU allocation.

For ALFWorld and WebShop, the primary UOPD schedules linearly reduce the target intervention
rate from $30\%$ to $5\%$ for the $3$B student and from $50\%$ to $30\%$ for the $1.5$B student.
The intervention threshold is calibrated online using teacher confidence on student-proposed actions.
The SFT weight, teacher execution, and intervention schedules are compared in
Section~\ref{subsec:exp-ablation}. Table~\ref{tab:app-train-hparams} lists the full configuration.

\paragraph{Intervention signals.}
For a student-proposed action $a_t^S$ at history $h_t$, the teacher-confidence
signal is the teacher's length-normalized negative log-likelihood:
\[
\delta_t^{\mathrm{conf}}
=-\frac{\log\pi_T(a_t^S\mid h_t)}{|a_t^S|}.
\]
The student--teacher confidence gap is
\[
\delta_t^{\mathrm{gap}}
=\frac{\log\pi_\theta(a_t^S\mid h_t)-\log\pi_T(a_t^S\mid h_t)}{|a_t^S|}.
\]
Here $|a_t^S|$ is the number of tokens in the action response, and
$\pi_\theta$ is the student policy used to generate the action.
Larger scores indicate lower teacher confidence or a larger student--teacher
confidence gap, respectively; both trigger intervention when they exceed
$\tau_n$.

\begin{table}[!ht]
\centering
\caption{\textbf{Training hyperparameters.} Settings shared by all methods are listed first; the
last block is specific to UOPD.}
\label{tab:app-train-hparams}
\scriptsize
\setlength{\tabcolsep}{4pt}
\begin{tabular}{@{}l c c c@{}}
\toprule
 & ALFWorld & WebShop & Search \\
\midrule
\multicolumn{4}{@{}l}{\emph{Optimisation}} \\
Optimizer & \multicolumn{3}{c}{AdamW, $\beta{=}(0.9, 0.999)$, weight decay $0.01$, constant learning rate, no warm-up} \\
Learning rate & $1\times10^{-6}$ & $1\times10^{-6}$ & $1\times10^{-6}$ \\
Gradient clipping & $1.0$ & $1.0$ & $1.0$ \\
Training steps & $250$ & $150$ & $300$ \\
Tasks per rollout batch & $16$ & $16$ & $64$ \\
Rollouts per task & $1$ & $1$ & $1$ \\
Update epochs / mini-batch & $1$ / $64$ & $1$ / $64$ & $1$ / $64$ \\
Dynamic batching (tokens per GPU) & $16{,}384$ & $16{,}384$ & $9{,}728$ \\
Sampling & \multicolumn{3}{c}{asynchronous, at most $2$ policy versions stale} \\
\midrule
\multicolumn{4}{@{}l}{\emph{Loss}} \\
Distillation (reverse-KL) coefficient $\lambda$ & $1.0$ & $1.0$ & $1.0$  \\
\midrule
\multicolumn{4}{@{}l}{\emph{Rollout}} \\
Max prompt / response tokens & $2{,}048$ / $512$ & $4{,}096$ / $512$ & $4{,}096$ / $512$\\
Sampling temperature & $1.0$ & $1.0$ & $1.0$ \\
Episode limit (steps / turns) & $50$ & $15$ & $8$ \\
History in the prompt & last $2$ steps & last $2$ steps & full episode \\
Retrieval & -- & -- & top $3$ passages \\
\midrule
\multicolumn{4}{@{}l}{\emph{UOPD}} \\
Intervention signal & Teacher confidence & Teacher confidence & Confidence gap \\
Scoring temperature for the signal & $1.0$ & $1.0$ & $1.0$ \\
Intervention schedule, $3$B & $30\%\!\to\!5\%$ over $120$ & $30\%\!\to\!5\%$ over $75$ & -- \\
Intervention schedule, $1.5$B & $50\%\!\to\!30\%$ over $120$ & $50\%\!\to\!30\%$ over $75$ & constant $30\%$ \\
SFT weight $\beta$, form & $1.0$ & $1.0$ & $1.0$\\
\bottomrule
\end{tabular}
\\[3pt]
\end{table}

\begin{table}[!ht]
\centering
\caption{\textbf{Eight-GPU allocation.}
Student rollout, teacher inference, and student updates run on separate GPU groups.}
\label{tab:app-gpu-allocation}
\small
\setlength{\tabcolsep}{5pt}
\begin{tabular}{@{}l c l@{}}
\toprule
Role & GPUs & Configuration \\
\midrule
Student rollout (vLLM) & $4$ & $2$ engines, tensor parallel size $2$ per engine \\
Teacher inference (vLLM) & $2$ & $1$ engine, tensor parallel size $2$ \\
Student updates (verl/FSDP) & $2$ & Ulysses sequence parallel size $2$ \\
\midrule
Total & $8$ & Single node \\
\bottomrule
\end{tabular}
\end{table}

\subsection{Evaluation Protocol}
\label{app:evaluation-protocol}

\paragraph{ALFWorld.} We evaluate household task completion on $140$ seen and $134$ unseen
tasks, reporting the percentage of successfully completed tasks in each split. Average turns
are computed separately for each split over all episodes, including failures.

\paragraph{WebShop.} We use the $1000$-product catalog and $128$ test tasks. Task score measures
how well the selected product satisfies the request; success rate measures full task completion.
Average turns include both successful and failed episodes.
For ALFWorld and WebShop, Table~\ref{tab:exp-agentic} reports means and standard deviations
over three evaluation seeds.

\paragraph{Search.} Figure~\ref{fig:search-15b-multihop} reports exact match on each of the four
multi-hop test sets of Search-R1 \citep{jin2025searchr1}: HotpotQA, 2WikiMultiHopQA, MuSiQue and
Bamboogle, $22{,}523$ questions in total. Average turns are computed over the same $22{,}523$
questions. A turn is one model response, the answering one included, so an episode that reaches the
limit and receives the forced answer counts nine turns.
Both OPD and UOPD use seed $42$, greedy decoding, forced answers, and an eight-turn limit.

\begin{center}\small
\begin{tabular}{@{}l c c c@{}}
\toprule
 & ALFWorld & WebShop & Search \\
\midrule
Temperature & $0.4$ & $0.4$ & $0$ (greedy) \\
Max prompt / response tokens & $2{,}048$ / $512$ & $4{,}096$ / $512$ & as in training / $512$ \\
Episode limit & $50$ & $15$ & $8$, then a forced answer \\
Test set & $140$ seen, $134$ unseen & $128$ test & $22{,}523$ multi-hop questions \\
\bottomrule
\end{tabular}
\end{center}

\noindent Every evaluation uses the checkpoint of the final training step. A Search episode that
reaches the turn limit without an answer is given one additional turn that asks for the final answer
(Appendix~\ref{app:prompt-search}); this turn is used only at evaluation.

\subsection{Hardware}

Training runs use a single node, primarily with eight NVIDIA RTX PRO 6000
Blackwell GPUs. We also use four-GPU setups with NVIDIA H100 (80\,GB) or
NVIDIA RTX A6000 (48\,GB) GPUs. All reported training-time comparisons use the
eight-GPU RTX PRO 6000 Blackwell setup, with the allocation in
Table~\ref{tab:app-gpu-allocation}.

\section{Environment Prompts}
\label{app:prompts}

This section gives the prompts used for each environment during training and evaluation, verbatim.
Braces mark the fields filled at each step. The student and the teacher receive the identical prompt,
which lets the teacher score the student's own tokens, and each prompt is sent as a single
chat-formatted user message. We pass no system message, so the model's chat template adds its default
one: ``You are Qwen, created by Alibaba Cloud. You are a helpful assistant.'' for the Qwen2.5-Instruct
students, and ``You are a helpful assistant.'' for the base Search student. No prompt states the step
or turn limit; the environment enforces it.

\subsection{ALFWorld Prompts}
\label{app:prompt-alfworld}

ALFWorld is a text-based household environment in which the agent completes object-manipulation
tasks. At each step the agent reasons inside \texttt{<think>} \texttt{</think>} tags and then gives one
admissible action inside \texttt{<action>} \texttt{</action>} tags, which is executed in the
environment. The first step has its own template without the task, which the first observation
already contains; every later step uses the second template.

\begin{promptbox}{ALFWorld Prompt Template (first step)}
You are an expert agent operating in the ALFRED Embodied Environment.
Your current observation is: {current_observation}
Your admissible actions of the current situation are: [{admissible_actions}].

Now it's your turn to take an action.
You should first reason step-by-step about the current situation. This reasoning process MUST be enclosed within <think> </think> tags.
Once you've finished your reasoning, you should choose an admissible action for current step and present it within <action> </action> tags.
\end{promptbox}

\begin{promptbox}{ALFWorld Prompt Template (later steps)}
You are an expert agent operating in the ALFRED Embodied Environment. Your task is to: {task_description}
Prior to this step, you have already taken {step_count} step(s). Below are the most recent {history_length} observations and the corresponding actions you took: {action_history}
You are now at step {current_step} and your current observation is: {current_observation}
Your admissible actions of the current situation are: [{admissible_actions}].

Now it's your turn to take an action.
You should first reason step-by-step about the current situation. This reasoning process MUST be enclosed within <think> </think> tags.
Once you've finished your reasoning, you should choose an admissible action for current step and present it within <action> </action> tags.
\end{promptbox}

\noindent \texttt{\{admissible\_actions\}} lists the environment's admissible commands except
\texttt{help}, each in single quotes, one per line. The history holds the most recent steps, one per
line, each as \texttt{[Observation \{i\}: '\{obs\}', Action \{i\}: '\{act\}']}, and
\texttt{\{history\_length\}} is the number shown: $\min(2, \text{steps taken})$.

\subsection{WebShop Prompts}
\label{app:prompt-webshop}

WebShop is a simulated e-commerce website in which the agent searches for, selects and buys a product
that matches a natural-language instruction. The first template is used for the first two steps and
the second from the third step on, always with the two most recent steps as history
($\texttt{\{history\_length\}}=2$, entries as in ALFWorld, each observation flattened to one line).

\begin{promptbox}{WebShop Prompt Template (first two steps)}
You are an expert autonomous agent operating in the WebShop e-commerce environment.
Your task is to: {task_description}.
Your current observation is: {current_observation}.
Your admissible actions of the current situation are:
[
{available_actions}
].

Now it's your turn to take one action for the current step.
You should first reason step-by-step about the current situation, then think carefully which admissible action best advances the shopping goal. This reasoning process MUST be enclosed within <think> </think> tags.
Once you've finished your reasoning, you should choose an admissible action for current step and present it within <action> </action> tags.
\end{promptbox}

\begin{promptbox}{WebShop Prompt Template (later steps)}
You are an expert autonomous agent operating in the WebShop e-commerce environment.
Your task is to: {task_description}.
Prior to this step, you have already taken {step_count} step(s). Below are the most recent {history_length} observations and the corresponding actions you took: {action_history}
You are now at step {current_step} and your current observation is: {current_observation}.
Your admissible actions of the current situation are:
[
{available_actions}
].

Now it's your turn to take one action for the current step.
You should first reason step-by-step about the current situation, then think carefully which admissible action best advances the shopping goal. This reasoning process MUST be enclosed within <think> </think> tags.
Once you've finished your reasoning, you should choose an admissible action for current step and present it within <action> </action> tags.
\end{promptbox}

\paragraph{Action format for WebShop.} \texttt{\{available\_actions\}} lists one action per line, of
two types:
\begin{itemize}
  \item \texttt{search[<query>]}: search the catalog with a text query; listed only when the page has a
    search bar.
  \item \texttt{click[<element>]}: click an element of the current page, such as a product, an option
    or a navigation button; one line per clickable element.
\end{itemize}

\subsection{Search Prompts}
\label{app:prompt-search}

Search is open-domain question answering with a retrieval tool. The prompt is Search-R1's original one,
on which the teacher was trained, reproduced verbatim.

\begin{promptbox}{Search Prompt Template (Search-R1's original prompt)}
Answer the given question. You must conduct reasoning inside <think> and </think> first every time you get new information. After reasoning, if you find you lack some knowledge, you can call a search engine by <search> query </search> and it will return the top searched results between <information> and </information>. You can search as many times as your want. If you find no further external knowledge needed, you can directly provide the answer inside <answer> and </answer>, without detailed illustrations. For example, <answer> Beijing </answer>. Question: {question}
\end{promptbox}

\noindent The episode so far is appended directly after the prompt: each earlier model output, cut after
its first complete \texttt{<search>}\dots\texttt{</search>} or \texttt{<answer>}\dots\texttt{</answer>},
followed by the environment's reply. After a search the reply is the top $3$ passages, on its own line.
An output with neither a complete \texttt{<search>} nor a complete \texttt{<answer>} receives an
invalid-action reply instead.

\begin{promptbox}{Retrieved Passages (appended after each search)}
<information>Doc 1: {passage}
Doc 2: {passage}
Doc 3: {passage}</information>
\end{promptbox}

\begin{promptbox}{Invalid-Action Reply}
Invalid action. After your reasoning, choose exactly one action: <search> your query </search> or <answer> your final answer </answer>.
\end{promptbox}

\noindent Generation stops at \texttt{</search>} or \texttt{</answer>}. A search continues the episode
and an answer ends it; the episode is scored by exact match of its last \texttt{<answer>}. As in
Search-R1, the prompt does not state the turn limit; the environment ends the episode after $8$ turns.
In training, an episode that reaches the limit without an answer scores $0$. At evaluation it receives
one more turn, whose prompt is the Search prompt above with the episode so far, followed on a new line
by the forced-answer instruction.

\begin{promptbox}{Forced Final Answer (evaluation only)}
You can no longer search. Based on the information above you MUST now give your final answer inside <answer> and </answer>. For example, <answer> Beijing </answer>.
\end{promptbox}

\FloatBarrier

\section{Intervention Dynamics during Training}
\label{app:telemetry}
\label{sec:analysis}

We show intervention dynamics on ALFWorld for the 3B student with a
$30\%\!\to\!5\%$ schedule, examining how often UOPD intervenes and
where corrections occur during training. The run uses teacher confidence,
$\beta=1$, and teacher action execution. The curves show unsmoothed batch
statistics from this training run. The horizontal axis is the rollout
batch index $n$ in Eq.~\ref{eq:intervention_schedule}; asynchronous rollout
collection and optimization use different counters. The schedule decays over
$120$ rollout batches.

The actual intervention rate is computed within each trajectory and then averaged
across the batch. For intervention position, we average the zero-based turn indices
of corrections within each trajectory, then average over trajectories with at
least one correction. Trajectories without corrections contribute zero to the
rate but are excluded from the position mean.

Figure~\ref{fig:alfworld-intervention-dynamics} shows a decreasing intervention
rate. The episode-averaged rate falls below the scheduled
target, particularly late in training. Part of this gap comes from how trajectories
are weighted: quantile calibration uses a window of action scores, whereas the
plotted rate gives each trajectory equal weight. Over the final $20$ batches,
the episode-averaged rate is approximately $2\%$, while averaging the batch-level
ratios of total corrections to total turns gives $4\%$, closer to the $5\%$
target. This difference indicates that corrections are more concentrated in
longer trajectories.

Mean intervention position remains variable and does not show a sustained shift
toward earlier turns. Only about $22\%$ of trajectories receive any correction in the final $20$
batches, so the position curve describes a small subset of the collected rollouts.
Together, these results show that declining intervention frequency need not imply
a uniform shift in correction position: UOPD continues to select turns according
to uncertainty within the trajectories encountered by the student.

\begin{figure}[!htbp]
\centering
\includegraphics[width=0.48\textwidth]{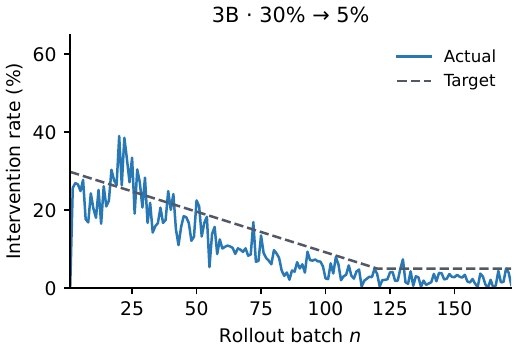}\hfill
\includegraphics[width=0.48\textwidth]{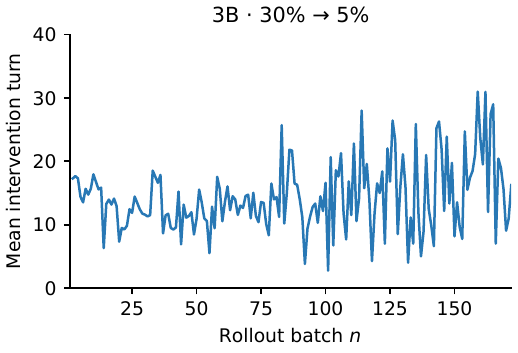}
\caption{\textbf{Intervention dynamics on ALFWorld with a 3B student and a $30\%\!\to\!5\%$ schedule.}
Left: actual episode-averaged intervention rate (blue) and scheduled target
(dashed gray). Right: mean intervention turn among corrected trajectories.
Interventions become less frequent, while their positions remain variable.}
\label{fig:alfworld-intervention-dynamics}
\end{figure}

\end{document}